\documentclass[twoside,11pt]{article}

\usepackage{jmlr2e}
\usepackage{times}
\usepackage{amsfonts}
\usepackage{amsmath}
\usepackage{amssymb}
\usepackage{amstext}
\usepackage{latexsym}
\usepackage{graphics}
\usepackage{enumerate}

\def\Rset{\mathbb{R}}

\def\Zset{\mathbb{Z}}

\DeclareMathOperator*{\E}{\rm E}

\DeclareMathOperator*{\argmin}{\rm argmin}
\providecommand{\abs}[1]{\lvert#1\rvert}
\providecommand{\norm}[1]{\lVert#1\rVert}

\newcommand{\ignore}[1]{}
\newcommand{\ipsfig}[2]{\scalebox{#1}{\epsfig{#2}}}
\newcommand{\set}[1]{\left\{ #1 \right\}}
\newcommand{\dotp}[2]{\left\langle #1 , #2 \right\rangle}
\newcommand{\h}{\widehat}
\newcommand{\tl}{\widetilde}
\renewcommand{\t}{\widetilde}
\newcommand{\sbeta}{\hat{\beta}}
\newcommand{\sbetam}{\gamma}

\newcommand{\bZ}{\mathbf Z}
\newcommand{\e}{\epsilon}
\newcommand{\nosmall}{\hspace{-.05in}}
\renewcommand{\phi}{\varphi}

\jmlrheading{1}{2000}{1-48}{4/00}{10/00}{Mehryar Mohri and Afshin Rostamizadeh}
\ShortHeadings{Stability Bounds for Non-i.i.d.\ Processes}{Mohri and Rostamizadeh}
\firstpageno{1}
\editor{TBD}

\begin{document}

\title{Stability Bounds for Stationary $\phi$-mixing and $\beta$-mixing Processes}

\author{\name Mehryar Mohri \email mohri@cims.nyu.edu\\
\addr Courant Institute of Mathematical Sciences\\
and Google Research\\
251 Mercer Street\\
New York, NY 10012\\
\AND
\name Afshin Rostamizadeh \email rostami@cs.nyu.edu\\
\addr Department of Computer Science\\
Courant Institute of Mathematical Sciences\\
251 Mercer Street\\
New York, NY 10012\\
}

\maketitle
\begin{abstract}
	Most generalization bounds in learning theory are based on some
	measure of the complexity of the hypothesis class used,
	independently of any algorithm. In contrast, the notion of
	algorithmic stability can be used to derive tight generalization
	bounds that are tailored to specific learning algorithms by
	exploiting their particular properties. However, as in much of
	learning theory, existing stability analyses and bounds apply only
	in the scenario where the samples are independently and identically
	distributed. In many machine learning applications, however, this
	assumption does not hold. The observations received by the learning
	algorithm often have some inherent temporal dependence.

  This paper studies the scenario where the observations are drawn
  from a stationary $\varphi$-mixing or $\beta$-mixing sequence, a
  widely adopted assumption in the study of non-i.i.d.\ processes that
  implies a dependence between observations weakening over time.  We
  prove novel and distinct stability-based generalization bounds for
  stationary $\varphi$-mixing and $\beta$-mixing sequences. These
  bounds strictly generalize the bounds given in the i.i.d.\ case and
  apply to all stable learning algorithms, thereby extending the
  use of stability-bounds to non-i.i.d.\ scenarios.

  We also illustrate the application of our $\varphi$-mixing
  generalization bounds to general classes of learning algorithms,
  including Support Vector Regression, Kernel Ridge Regression, and
  Support Vector Machines, and many other kernel regularization-based
  and relative entropy-based regularization algorithms.  These novel
  bounds can thus be viewed as the first theoretical basis for the use
  of these algorithms in non-i.i.d.\ scenarios.\\
\end{abstract}

\ignore{
\begin{abstract}
	Most generalization bounds in learning theory are based on some measure
	of the complexity of the hypothesis class used. In contrast, the notion
	of algorithmic stability can be used to derive tight generalization
	bounds that are tailored to specific learning algorithms. However, as in
	much of learning theory, existing stability analyses and bounds apply
	only in the scenario where the samples are independently and identically
	distributed. In many machine learning applications, however, this
	assumption does not hold.  The observations received by the learning
	algorithm often have some inherent temporal dependence.
	
	This paper studies the scenario where the observations are drawn from a
	stationary $\varphi$-mixing or $\beta$-mixing sequence, which implies a
	dependence between observations weakening over time.  We prove novel and
	distinct stability-based generalization bounds for stationary
	$\varphi$-mixing and $\beta$-mixing sequences. These bounds strictly
	generalize the bounds given in the i.i.d.\ case and apply to all stable
	learning algorithms.
	
	We also illustrate the application of our $\varphi$-mixing
	generalization bounds to general classes of learning algorithms,
	including SVR, KRR, and SVMs, and many other kernel regularization-based
	and relative entropy-based regularization algorithms.  These novel bounds
	can thus be viewed as the first theoretical basis for the use of these
	algorithms in non-i.i.d.\ scenarios.\\
\end{abstract}
}


\section{Introduction}
\label{sec:introduction}

Most generalization bounds in learning theory are based on some
measure of the complexity of the hypothesis class used, such as the
VC-dimension, covering numbers, or Rademacher complexity. These
measures characterize a class of hypotheses, independently of any
algorithm.  In contrast, the notion of algorithmic stability can be
used to derive bounds that are tailored to specific learning
algorithms and exploit their particular properties. A learning
algorithm is stable if the hypothesis it outputs varies in a limited
way in response to small changes made to the training set. Algorithmic
stability has been used effectively in the past to derive tight
generalization bounds \citep{bousquet,bousquet-jmlr}.

But, as in much of learning theory, existing stability analyses and
bounds apply only in the scenario where the samples are independently
and identically distributed (i.i.d.). In many machine learning
applications, this assumption, however, does not hold; in fact, the
i.i.d.\ assumption is not tested or derived from any data analysis.
The observations received by the learning algorithm often have some
inherent temporal dependence. This is clear in system diagnosis or
time series prediction problems. Clearly, prices of different stocks
on the same day, or of the same stock on different days, may be
dependent. But, a less apparent time dependency may affect data
sampled in many other tasks as well.

This paper studies the scenario where the observations are drawn from
a stationary $\varphi$-mixing or $\beta$-mixing sequence, a widely
adopted assumption in the study of non-i.i.d.\ processes that implies
a dependence between observations weakening over time
\citep{yu,meir,vidyasagar,lozano}.  We prove novel and distinct
stability-based generalization bounds for stationary $\varphi$-mixing
and $\beta$-mixing sequences. These bounds strictly generalize the
bounds given in the i.i.d.\ case and apply to all stable learning
algorithms, thereby extending the usefulness of stability-bounds to
non-i.i.d.\ scenarios.  Our proofs are based on the independent block
technique described by \citet{yu} and attributed to \citet{bernstein},
which is commonly used in such contexts. However, our analysis differs
from previous uses of this technique in that the blocks of points
considered are not of equal size. 

For our analysis of stationary $\varphi$-mixing sequences, we make use
of a generalized version of McDiarmid's inequality \citep{leo} that
holds for $\varphi$-mixing sequences. This leads to stability-based
generalization bounds with the standard exponential form. Our
generalization bounds for stationary $\beta$-mixing sequences cover a
more general non-i.i.d.\ scenario and use the standard McDiarmid's
inequality, however, unlike the $\varphi$-mixing case, the
$\beta$-mixing bound presented here is not a purely exponential bound
and contains an additive term depending on the mixing coefficient.
 
We also illustrate the application of our $\varphi$-mixing
generalization bounds to general classes of learning algorithms,
including Support Vector Regression (SVR) \citep{vapnik98}, Kernel
Ridge Regression \citep{krr}, and Support Vector Machines (SVMs)
\citep{ccvv}. Algorithms such as support vector regression (SVR)
\citep{vapnik98,smbook} have been used in the context of time series
prediction in which the i.i.d.\ assumption does not hold, some with
good experimental results \citep{muller97,mattera}. To our knowledge,
the use of these algorithms in non-i.i.d.\ scenarios has not been
previously supported by any theoretical analysis. The stability bounds
we give for SVR, SVMs, and many other kernel regularization-based and
relative entropy-based regularization algorithms can thus be viewed as
the first theoretical basis for their use in such scenarios.

The following sections are organized as follows. In
Section~\ref{sec:preliminaries}, we introduce the necessary
definitions for the non-i.i.d.\ problems that we are considering and
discuss the learning scenarios in that
context. Section~\ref{sec:non-iid-bounds} gives our main
generalization bounds for stationary $\varphi$-mixing sequences based
on stability, as well as the illustration of its applications to
general kernel regularization-based algorithms, including SVR, KRR,
and SVMs, as well as to relative entropy-based regularization
algorithms.  Finally, Section~\ref{sec:beta-mixing} presents the first
known stability bounds for the more general stationary $\beta$-mixing
scenario.

\section{Preliminaries}
\label{sec:preliminaries}

We first introduce some standard definitions for dependent observations
in mixing theory \citep{doukhan} and then briefly discuss the learning
scenarios in the non-i.i.d.\ case.

\subsection{Non-i.i.d.\ Definitions}
\label{sec:non-iid-definitions}

\begin{definition}
A sequence of random variables $\bZ = \set{Z_t}_{t = -\infty}^\infty$
is said to be \emph{stationary} if for any $t$ and non-negative integers
$m$ and $k$, the random vectors $(Z_t, \ldots, Z_{t + m})$ and $(Z_{t
+ k}, \ldots, Z_{t + m + k})$ have the same distribution.
\end{definition}
Thus, the index $t$ or time, does not affect the distribution of a
variable $Z_t$ in a stationary sequence. This does not imply
independence however. In particular, for $i < j < k$,
$\Pr[Z_j \mid Z_i]$ may not equal  $\Pr[Z_k \mid Z_i]$. The following is a
standard definition giving a measure of the dependence of the random
variables $Z_t$ within a stationary sequence.  There are several
equivalent definitions of this quantity, we are adopting here that of
\citep{yu}.\\

\begin{definition}
Let $\bZ = \set{Z_t}_{t = -\infty}^\infty$ be a stationary sequence of
random variables. For any $i, j \in \Zset \cup \set{-\infty,
+\infty}$, let $\sigma_i^j$ denote the $\sigma$-algebra generated by
the random variables $Z_k$, $i \leq k \leq j$. Then, for any positive
integer $k$, the $\beta$-mixing and $\varphi$-mixing coefficients of
the stochastic process $\bZ$ are defined as
\begin{equation}
\beta(k) = \sup_{n} \E_{B \in \sigma_{-\!\infty}^n} \Bigl[\sup_{A \in
\sigma_{n + k}^\infty} \Bigl|\Pr[A \mid B] - \Pr[A]
\Bigr| \Bigr]
\quad
\varphi(k) = \sup_{\substack{n\\ A \in
\sigma_{n + k}^\infty\\ B \in \sigma_{-\!\infty}^{n}}} \Bigl|\Pr[A \mid B] - \Pr[A]
\Bigr|.
\end{equation}
$\bZ$ is said to be $\beta$-mixing ($\varphi$-mixing) if $\beta(k) \to
0$ (resp. $\varphi(k) \to 0$) as $k \to \infty$. It is said to be
\emph{algebraically $\beta$-mixing} (\emph{algebraically
$\varphi$-mixing}) if there exist real numbers $\beta_0 > 0$
(resp. $\varphi_0 > 0$) and $r > 0$ such that $\beta(k) \leq \beta_0/
k^{r}$ (resp. $\varphi(k) \leq \varphi_0/ k^{r}$) for all $k$,
\emph{exponentially mixing} if there exist real numbers $\beta_0$
(resp. $\varphi_0 > 0$) and $\beta_1$ (resp. $\varphi_1 > 0$) such
that $\beta(k) \leq \beta_0 \exp(-\beta_1 k^r)$ (resp. $\varphi(k) \leq
\varphi_0 \exp(-\varphi_1 k^r)$) for all $k$.
\end{definition}
Both $\beta(k)$ and $\varphi(k)$ measure the dependence of an event on
those that occurred more than $k$ units of time in the
past. $\beta$-mixing is a weaker assumption than $\phi$-mixing and
thus covers a more general non-i.i.d.\ scenario.

This paper gives stability-based generalization bounds both in the
$\varphi$-mixing and $\beta$-mixing case. The $\beta$-mixing bounds
cover a more general case of course, however, the $\varphi$-mixing
bounds are simpler and admit the standard exponential form. The
$\varphi$-mixing bounds are based on a concentration inequality that
applies to $\phi$-mixing processes only. Except from the use of this
concentration bound, all of the intermediate proofs and results to
derive a $\phi$-mixing bound in Section~\ref{sec:non-iid-bounds} are
given in the more general case of $\beta$-mixing sequences.

It has been argued by \citet{vidyasagar} that $\beta$-mixing is ``just
the right'' assumption for the analysis of weakly-dependent sample
points in machine learning, in particular because several PAC-learning
results then carry over to the non-i.i.d.\ case. Our $\beta$-mixing
generalization bounds further contribute to the analysis of this
scenario.\footnote{Some results have also been obtained in the more general
context of $\alpha$-mixing but they seem to require the stronger
condition of exponential mixing \citep{modha}.}

We describe in several instances the application of our bounds in
the case of algebraic mixing. Algebraic mixing is a standard
assumption for mixing coefficients that has been adopted in previous
studies of learning in the presence of dependent observations
\citep{yu,meir,vidyasagar,lozano}.

Let us also point out that mixing assumptions can be checked in some
cases such as with Gaussian or Markov processes \citep{meir} and that
mixing parameters can also be estimated in such cases.

Most previous studies use a technique originally introduced by
\citet{bernstein} based on \emph{independent blocks} of equal size
\citep{yu,meir,lozano}. This technique is particularly relevant when
dealing with stationary $\beta$-mixing. We will need a related but
somewhat different technique since the blocks we consider may not have
the same size. The following lemma is a special case of Corollary 2.7
from \citep{yu}.\\

\begin{lemma}[Yu \citep{yu}, Corollary 2.7]
\label{lemma:blocks}
Let $\mu \geq 1$ and suppose that $h$ is measurable function, with 
absolute value bounded by $M$,
on a product probability space $\left(\prod_{j=1}^\mu \Omega_j,
\prod_{i=1}^\mu \sigma_{r_i}^{s_i}\right)$ where $r_i \leq s_i \leq
r_{i+1}$ for all $i$.  Let $Q$ be a probability measure on the product
space with marginal measures $Q_i$ on ($\Omega_i,
\sigma_{r_i}^{s_i}$), and let $Q^{i+1}$ be the marginal measure of $Q$
on $\left(\prod_{j=1}^{i+1} \Omega_j, \prod_{j=1}^{i+1}
\sigma_{r_j}^{s_j}\right)$, $i = 1, \ldots, \mu - 1$. Let $\beta(Q) =
\sup_{1 \leq i \leq \mu-1} \beta(k_i)$, where $k_i = r_{i+1} - s_{i}$,
and $P = \prod_{i=1}^\mu Q_i$. Then,
\begin{equation}
| \E_Q[h] - \E_P[h] | \leq (\mu - 1) M \beta(Q).
\end{equation}
\end{lemma}
The lemma gives a measure of the difference between the distribution
of $\mu$ blocks where the blocks are independent in one case and
dependent in the other case.
The distribution within each block is assumed to be the 
same in both cases. For a monotonically decreasing function $\beta$,
we have $\beta(Q) = \beta(k^*)$, where $k^* = \min_i(k_i)$ is the
smallest gap between blocks.

\subsection{Learning Scenarios}
\label{sec:scenarios}

We consider the familiar supervised learning setting where the
learning algorithm receives a sample of $m$ labeled points $S = (z_1,
\ldots, z_m) = ((x_1, y_1), \ldots, (x_m, y_m)) \in (X \times Y)^m$,
where $X$ is the input space and $Y$ the set of labels ($Y = \Rset$ in
the regression case), both assumed to be measurable. 

For a fixed learning algorithm, we denote by $h_S$ the hypothesis it
returns when trained on the sample $S$. The error of a hypothesis on a
pair $z \in X \times Y$ is measured in terms of a cost function $c: Y
\times Y \to \Rset_+$. Thus, $c(h(x), y)$ measures the error of a
hypothesis $h$ on a pair $(x, y)$, $c(h(x), y) = (h(x) - y)^2$ in the
standard regression cases. We will use the shorthand $c(h, z) :=
c(h(x), y)$ for a hypothesis $h$ and $z = (x, y) \in X \times Y$ and
will assume that $c$ is upper bounded by a constant $M > 0$. We denote
by $\h R(h)$ the empirical error of a hypothesis $h$ for a training
sample $S = (z_1, \ldots, z_m)$:
\begin{equation}
\widehat{R}(h) = \frac{1}{m} \sum_{i=1}^m c(h,z_i).
\end{equation}
In the standard machine learning scenario, the sample pairs $z_1,
\ldots, z_m$ are assumed to be i.i.d., a restrictive assumption that
does not always hold in practice. We will consider here the more
general case of dependent samples drawn from a stationary mixing
sequence $\bZ$ over $X \times Y$. As in the i.i.d.\ case, the
objective of the learning algorithm is to select a hypothesis with
small error over future samples. But, here, we must distinguish two
versions of this problem.

In the most general version, future samples depend on the training
sample $S$ and thus the generalization error or true error of the
hypothesis $h_S$ trained on $S$ must be measured by its expected error
conditioned on the sample $S$:
\begin{equation}
\label{eq:risk}
R(h_S) = \E_z[c(h_S, z) \mid S].
\end{equation}
This is the most realistic setting in this context, which matches time
series prediction problems. A somewhat less realistic version is one
where the samples are dependent, but the test points are assumed to
be independent of the training sample $S$. The generalization error of
the hypothesis $h_S$ trained on $S$ is then:
\begin{equation}
R(h_S) = \E_z[c(h_S, z) \mid S] = \E_z[c(h_S, z)].
\end{equation}
This setting seems less natural since, if samples are dependent,
future test points must also depend on the training points, even if
that dependence is relatively weak due to the time interval after
which test points are drawn. Nevertheless, it is this somewhat less
realistic setting that has been studied by all previous machine
learning studies that we are aware of
\citep{yu,meir,vidyasagar,lozano}, even when examining specifically a
time series prediction problem \citep{meir}. Thus, the bounds derived
in these studies cannot be directly applied to the more general
setting.

We will consider instead the most general setting with the definition
of the generalization error based on Eq.~\ref{eq:risk}. Clearly, our
analysis also applies to the less general setting just discussed as
well.

Let us briefly discuss the more general scenario of
\emph{non-stationary} mixing sequences, that is one where the
distribution may change over time. Within that general case, the
generalization error of a hypothesis $h_S$, defined straightforwardly
by
\begin{equation}
	R(h_S, t) = \E_{z_t \sim \sigma_{t}^t} [c(h_S,z_t) \mid S],
\end{equation}
would depend on the time $t$ and it may be the case that $R(h_S, t)
\neq R(h_S, t')$ for $t \neq t'$, making the definition of the
generalization error a more subtle issue.  To remove the dependence on
time, one could define a weaker notion of the generalization error
based on an expected loss over all time:
\begin{equation}
	R(h_S) = \E_t[R(h_S,t)].
\end{equation}
It is not clear however whether this term could be easily computed and
useful. A stronger condition would be to minimize the generalization
error for any particular target time. Studies of this type have been
conducted for smoothly changing distributions, such as in
\citet{lafferty}, however, to the best of our knowledge, the scenario
of a both non-identical and non-independent sequences has not yet been
studied.

\section{$\phi$-Mixing Generalization Bounds and Applications}
\label{sec:non-iid-bounds}

This section gives generalization bounds for $\sbeta$-stable algorithms
over a mixing stationary distribution.\footnote{The standard variable used
for the stability coefficient is $\beta$. To avoid the confusion with the
$\beta$-mixing coefficient, we will use $\sbeta$ instead.} The first two
sections present our main proofs which hold for $\beta$-mixing stationary
distributions. In the third section, we will briefly discuss concentration
inequalities that apply to $\phi$-mixing processes only. Then, in the final
section, we will present our main results.

The condition of $\sbeta$-stability is an algorithm-dependent property
first introduced by \citet{devroye} and \citet{kearns}. It has been
later used successfully by \citet{bousquet,bousquet-jmlr} to show
algorithm-specific stability bounds for i.i.d.\ samples. Roughly
speaking, a learning algorithm is said to be \emph{stable} if small
changes to the training set do not produce large deviations in its
output. The following gives the precise technical definition.

\begin{definition}
A learning algorithm is said to be (uniformly) \emph{$\sbeta$-stable}
if the hypotheses it returns for any two training samples $S$ and $S'$
that differ by a single point satisfy
\begin{equation}
\forall z \in X \times Y, \quad |c(h_S, z) - c(h_{S'}, z)| \leq \sbeta.
\end{equation}
\end{definition}
The use of stability in conjunction with McDiarmid's inequality will allow
us to produce generalization bounds.  McDiarmid's inequality is an
exponential concentration bound of the type,
\begin{equation*}
	\Pr[ | \Phi - \E[\Phi] | \geq \e ] \leq \exp \left(-\frac{
	\e^2}{m l^2} \right),
\end{equation*}
where the probability is over a sample of size $m$ and $l$ is the
Lipschitz parameter of $\Phi$ (which is also a function of $m$).
Unfortunately, this inequality cannot be easily applied when the
sample points are not distributed in an i.i.d.\ fashion.  We will use
the results of \citet{leo} to extend the use of McDiarmid's inequality
with general mixing distributions (Theorem~\ref{th:leo}).

To obtain a stability-based generalization bound, we will apply this
theorem to $\Phi(S) = R(h_S) - \h R(h_S)$. To do so, we need to show,
as with the standard McDiarmid's inequality, that $\Phi$ is a
Lipschitz function and, to make it useful, bound $\E[\Phi]$.  The next
two sections describe how we achieve both of these in this non-i.i.d.\
scenario.

Let us first take a brief look at the problem faced when attempting to
give stability bounds for dependent sequences and give some idea of
our solution for that problem. The stability proofs given by
\citet{bousquet} assume the i.i.d.\ property, thus replacing an
element in a sequence with another does not affect the expected value
of a random variable defined over that sequence.  In other words, the
following equality holds,
\begin{equation}
	\label{eq:iid_property}
        \E_S[V(Z_1, \ldots,Z_i, \ldots, Z_m)] =
	\E_{S,Z'}[V(Z_1, \ldots,Z', \ldots, Z_m)],
\end{equation} 
for a random variable $V$ that is a function of the sequence of random
variables $S = (Z_1, \ldots, Z_m)$.  However, clearly, if the points in
that sequence $S$ are dependent, this equality may not hold
anymore.

The main technique to cope with this problem is based on the so-called
``independent block sequence'' originally introduced by
\citet{bernstein}. \ignore{This consists of examining, instead of the
original dependent sequence, independent blocks of points, each with
the same distribution as the block of the same size within the
original sequence. }This consists of eliminating from the original
dependent sequence several blocks of contiguous points, leaving us
with some remaining blocks of points.  Instead of these dependent
blocks, we then consider independent blocks of points, each with the
same size and the same distribution (within each block) as the
dependent ones. By Lemma~\ref{lemma:blocks}, for a $\beta$-mixing
distribution, the expected value of a random variable defined over the
dependent blocks is close to the one based on these independent blocks.
Working with these independent blocks brings us back to a situation
similar to the i.i.d.\ case, with i.i.d.\ blocks replacing i.i.d.\
points.

Our use of this method somewhat differs from previous ones
\citep[see][]{yu,meir} where many blocks of equal size are considered.
We will be dealing with four blocks and with typically unequal sizes.
More specifically, note that for Equation~\ref{eq:iid_property} to
hold, we only need that the variable $Z_i$ be independent of the other
points in the sequence.  To achieve this, roughly speaking, we will be
``discarding'' some of the points in the sequence surrounding $Z_i$.
This results in a sequence of three blocks of contiguous points.  If
our algorithm is stable and we do not discard too many points, the
hypothesis returned should not be greatly affected by this operation.
In the next step, we apply the independent block lemma, which then
allows us to assume each of these blocks as independent modulo the
addition of a mixing term. In particular, $Z_i$ becomes independent of
all other points. Clearly, the number of points discarded is subject
to a trade-off: removing too many points could excessively modify the
hypothesis returned; removing too few would maintain the dependency
between $Z_i$ and the remaining points, thereby producing a larger
penalty when applying Lemma~\ref{lemma:blocks}. This trade-off is made
explicit in the following section where an optimal solution is sought.

\subsection{Lipschitz Bound}
\label{sec:lipschitz}

As discussed in Section~\ref{sec:scenarios}, in the most general
scenario, test points depend on the training sample. We first present
a lemma that relates the expected value of the generalization error in
that scenario and the same expectation in the scenario where the test
point is independent of the training sample. We denote by $R(h_S) =
\E_z[c(h_S, z) | S]$ the expectation in the dependent case and by $\t
R(h_{S_b}) = \E_{\t z}[c(h_{S_b}, \t z)]$ the expectation where the
test points are assumed independent of the training, with $S_b$
denoting a sequence similar to $S$ but with the last $b$ points
removed. Figure~\ref{fig:S}(a) illustrates that sequence. The block
$S_b$ is assumed to have exactly the same distribution as the
corresponding block of the same size in $S$.

\begin{lemma}
\label{lemma:risk}
Assume that the learning algorithm is $\sbeta$-stable and that the
cost function $c$ is bounded by $M$. Then, for any sample $S$ of size
$m$ drawn from a $\beta$-mixing stationary distribution and for any $b
\in \{0,\ldots, m\}$, the following holds:
\begin{equation}
|\E_S[R(h_S)] - \E_S[\t R(h_{S_b})]| \leq b \sbeta + \beta(b) M.
\end{equation}
\end{lemma}
\begin{proof}
The $\sbeta$-stability of the learning algorithm implies that
\begin{equation}
\E_S[R(h_S)] = \E_{S,z} [ c(h_S, z) ] \leq \E_{S,z} [ c(h_{S_b}, z) ]
+ b \sbeta.
\end{equation}
The application of Lemma~\ref{lemma:blocks} yields
\begin{equation}
\E_S[R(h_S)] \leq \E_{S,\t z} [ c(h_{S_b}, \t z) ] + b \sbeta +
\beta(b) M = \t \E_S [ R(h_{S_b}) ] + b \sbeta + \beta(b) M.
\end{equation}
The other side of the inequality of the lemma can be shown following
the same steps.
\end{proof}
\ignore{
Notice that $\sbeta$-stability provides an upper bound for replacing a training
point, which is an upper bound for removing a point (**is this obvious, or
should there be more discussion?). Thus, the first inequality is true by
stability of the algorithm that produces $h_S$ and $h_{S_b}$.  The second
inequality is a direct application of Lemma ?, which is introduced in the
next section.
}
We can now prove a Lipschitz bound for the function $\Phi$.
\begin{lemma}
\label{lemma:phi-bound}
Let $S = (z_1, \dots, z_i, \dots, z_m)$ and $S^i = (z_1, \dots, z_i',
\dots, z_m)$ be two sequences drawn from a $\beta$-mixing stationary
process that differ only in point $i \in [1, m]$, and let $h_S$ and
$h_{S^i}$ be the hypotheses returned by a $\sbeta$-stable algorithm
when trained on each of these samples. Then, for any $i \in [1, m]$,
the following inequality holds:
\begin{equation}
|\Phi(S) - \Phi(S^i)| \leq (b + 1) 2 \sbeta + 2 \beta(b) M + \frac{M}{m}.
\end{equation}
\end{lemma}
\begin{proof}
To prove this inequality, we first bound the difference of the
empirical errors as in \citep{bousquet-jmlr}, then the difference of the
true errors. Bounding the difference of costs on agreeing points with
$\sbeta$ and the one that disagrees with $M$ yields
\begin{eqnarray}
|\h R(h_S) - \h R(h_{S^i})| & 
 = & \frac{1}{m} \sum_{j \neq i} |c(h_S,z_j) -
c(h_{S^i},z_j)| + \frac{1}{m} |c(h_S,z_i) - c(h_{S^i},z_i')| 
\label{eq:e1} \\
& \leq & \sbeta
+ \frac{M}{m}.\nonumber
\end{eqnarray}
Since both $R(h_S)$ and $R(h_{S^i})$ are defined with respect to a
(different) dependent point, we apply Lemma~\ref{lemma:risk} to both
generalization error terms and use $\sbeta$-stability. This then
results in
\begin{eqnarray}
|R(h_S) - R(h_{S^i})| & \leq & |\t R(h_{S_b}) - \t R(h_{S_b^i})| 
+ 2 b \sbeta + 2 \beta(b) \label{eq:e2}\\
& = & \E_{\t z} [ c(h_{S_b}, \t z) - c(h_{S_b^i}, \t z) ] + 2 b \sbeta 
+ 2 \beta(b) M  \leq  \sbeta + 2 b \sbeta + 2 \beta(b) M. \nonumber
\end{eqnarray}
The lemma's statement is obtained by combining
inequalities~\ref{eq:e1} and \ref{eq:e2}.
\end{proof}

\begin{figure}[t]
\begin{center}
\begin{tabular}{ll}
\ipsfig{.325}{figure=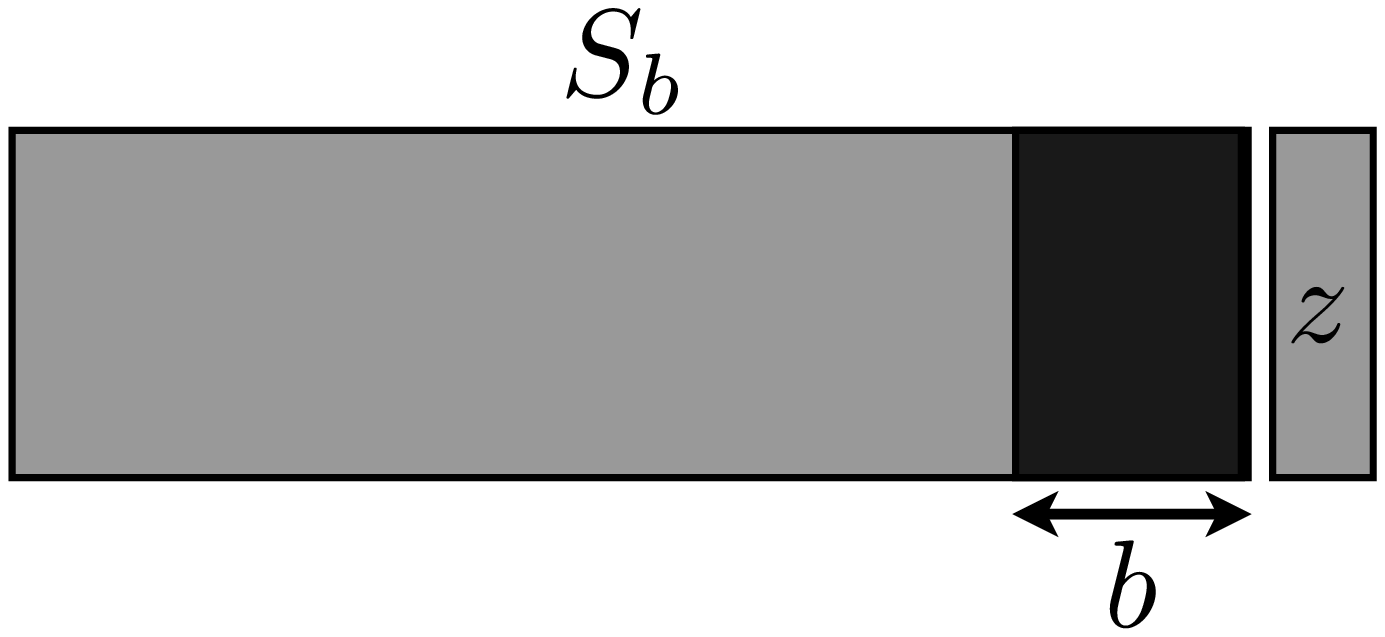} &
\ipsfig{.325}{figure=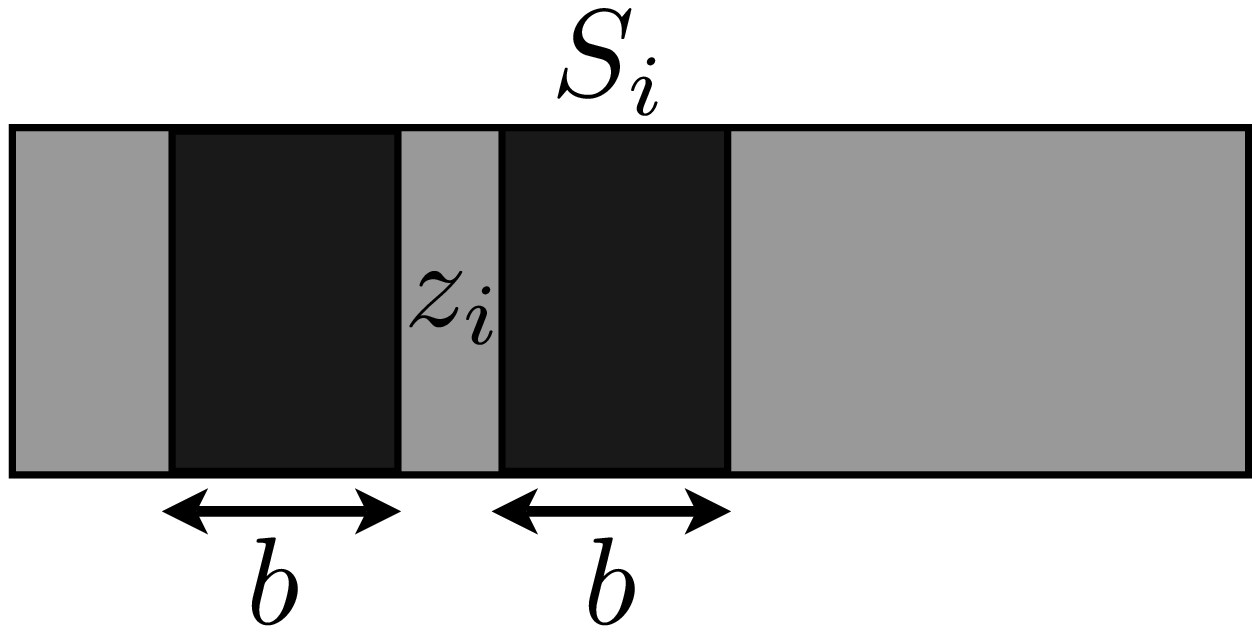}  \\
\multicolumn{1}{c}{\small (a)} & \multicolumn{1}{c}{\small (b)} \\
\ipsfig{.325}{figure=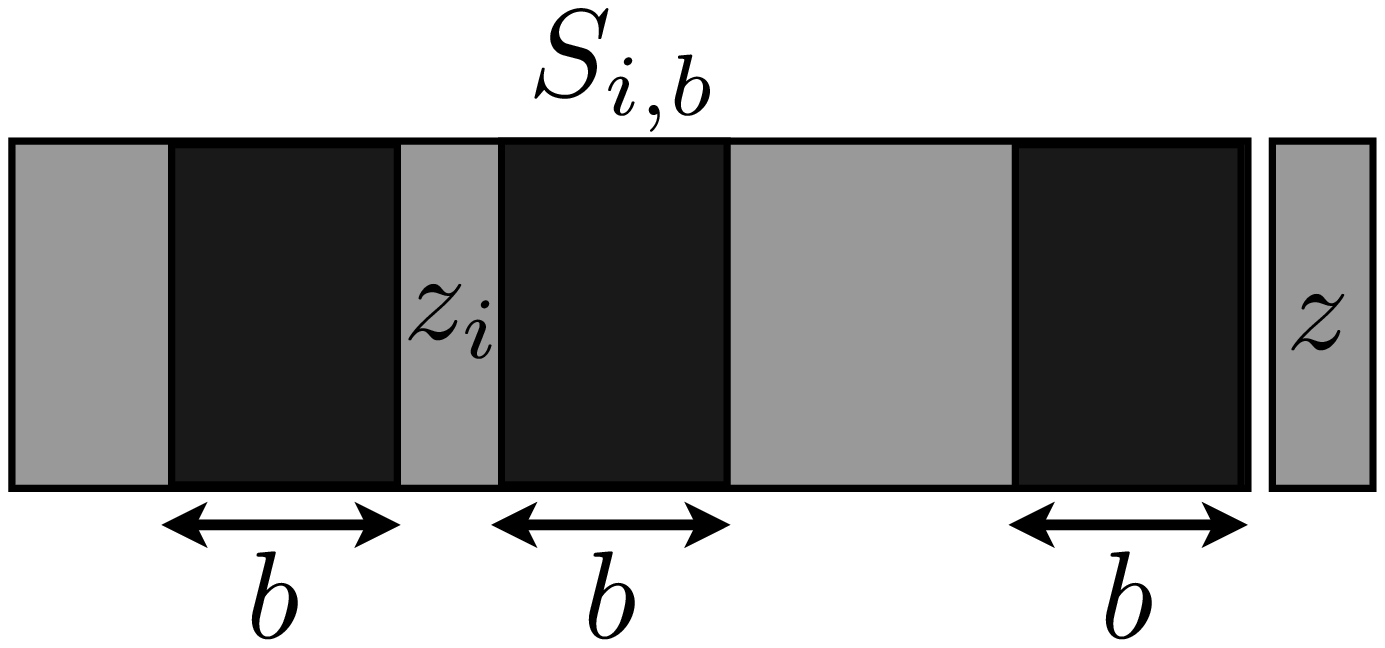} & 
\ipsfig{.325}{figure=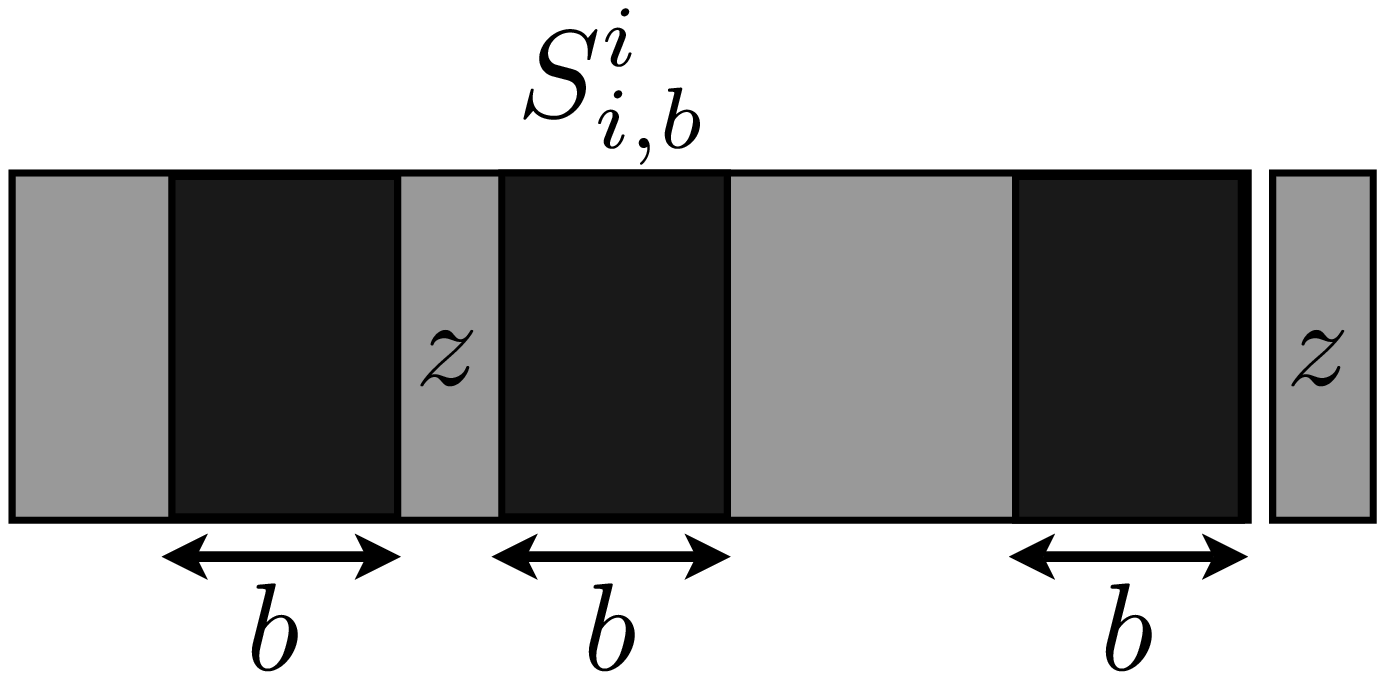}\\
\multicolumn{1}{c}{\small (c)} & \multicolumn{1}{c}{\small (d)}
\end{tabular}
\end{center}
\caption{Illustration of the sequences derived from $S$ that are
considered in the proofs.}
\label{fig:S}
\end{figure}

\subsection{Bound on Expectation}
\label{sec:bound-on-expectation}

As mentioned earlier, to obtain an explicit bound after application of
a generalized McDiarmid's inequality, we also need to bound
$\E_S[\Phi(S)]$\ignore{ and determine the constant $l$ for which
$\Phi$ is $l$-Lipschitz}. This is done by analyzing independent blocks
using Lemma~\ref{lemma:blocks}.

\begin{lemma}
\label{lemma:E-bound}
Let $h_S$ be the hypothesis returned by a $\sbeta$-stable algorithm
trained on a sample $S$ drawn from a stationary $\beta$-mixing
distribution. Then, for all $b \in [1, m]$, the following inequality
holds:
\begin{equation}
\E_S[|\Phi(S)|] \leq (6b + 1) \sbeta + 3 \beta(b) M.
\end{equation}
\end{lemma}
\begin{proof}
\ignore{
We first analyze the term $\E_S[\h R(h_S)]$. Let $S_i$ be the sequence
$S$ with the $b$ points before and after point $z_i$
removed. Figure~\ref{fig:S}(b) illustrates this definition. $S_i$ is
thus made of three blocks. Let $\t S_i$ denote a similar set of three
blocks each with the same distribution as the corresponding block in
$S_i$, but such that the three blocks are independent. In particular,
the middle block reduced to one point $\t z_i$ is independent of the
two others. By the $\sbeta$-stability of the algorithm, 
\begin{equation}
\E_S [\h R(h_S)] = \E_S \left [ \frac{1}{m} \sum_{i=1}^m c(h_S, z_i) \right ] 
\leq \E_{S_i} \left [ \frac{1}{m} \sum_{i=1}^m c(h_{S_i}, z_i)  \right ] 
+ 2 b \sbeta.
\end{equation}
Applying Lemma~\ref{lemma:blocks} to the first term of the right-hand
side yields
\begin{equation}
\E_S [\h R(h_S)] \leq \E_{\t S_i} \left [ \frac{1}{m} \sum_{i=1}^m
c(h_{\t S_i}, \t z_i) \right ] + 2 b \sbeta + 2 \beta(b) M.
\end{equation}
Combining the independent block sequences associated to $\h R(h_S)$
and $R(h_S)$ will help us prove the lemma in a way similar to the
i.i.d.\ case treated in \citep{bousquet-jmlr}. 
}
Let $S_b$ be defined as in the proof of Lemma~\ref{lemma:risk}. To
deal with independent block sequences defined with respect to the same
hypothesis, we will consider the sequence $S_{i, b} = S_i \cap S_b$,
which is illustrated by Figure~\ref{fig:S}(c). This can result in as
many as four blocks.  As before, we will consider a sequence $\t S_{i,
b}$ with a similar set of blocks each with the same distribution as
the corresponding blocks in $S_{i, b}$, but such that the blocks are
independent.

Since three blocks of at most $b$ points are removed from each
hypothesis, by the $\sbeta$-stability of the learning algorithm,
the following holds:
\begin{eqnarray}
\E_S [\Phi(S)] & = & \E_S [\h R(h_S) - R(h_S)] = \E_{S,z} \left [ \frac{1}{m} \sum_{i=1}^m c(h_S, z_i) - c(h_S, z) \right] \\
& \leq & \E_{S_{i,b}, z} \left[ \frac{1}{m} \sum_{i = 1}^m
  c(h_{S_{i,b}}, z_i) - c(h_{S_{i,b}}, z) \right] + 6 b \sbeta.
\end{eqnarray}
The application of Lemma~\ref{lemma:blocks} to the difference of
two cost functions also bounded by $M$ as in the right-hand side leads
to
\begin{equation}
\E_S [\Phi(S)] \leq \E_{\t S_{i,b}, \t z} \left [ \frac{1}{m}
    \sum_{i=1}^m c(h_{\t S_{i,b}}, \t z_i) - c(h_{\t S_{i,b}}, \t z)
    \right] + 6 b \sbeta + 3 \beta(b) M.
\end{equation}
Now, since the points $\t z$ and $\t z_i$ are independent and since
the distribution is stationary, they have the same distribution and we
can replace $\t z_i$ with $\t z$ in the empirical cost. Thus, we can write
\begin{equation*}
\E_S [\Phi(S)] \leq \E_{\t S_{i,b}, \t z} \left [ \frac{1}{m}
\sum_{i=1}^m c(h_{\t S_{i,b}^i}, \t z) - c(h_{\t S_{i,b}}, \t z)
\right ] + 6 b \sbeta + 3 \beta(b) M \leq  \sbeta + 6 b \sbeta + 3 \beta(b) M,
\end{equation*}
where $\t S_{i,b}^i$ is the sequence derived from $\t S_{i,b}$ by
replacing $\t z_i$ with $\t z$. The last inequality holds by
$\sbeta$-stability of the learning algorithm. The other side of the
inequality in the statement of the lemma can be shown 
following the same steps.
\end{proof}

\subsection{$\varphi$-mixing Generalization Bounds}
\label{sec:phi-generalization-bounds}

We are now prepared to make use of a concentration inequality to
provide a generalization bound in the $\varphi$-mixing scenario.
Several concentration inequalities have been shown in $\varphi$-mixing
case, e.g.\ \citet{marton,samson,chazottes,leo}.  We will
use that of \citet*{leo}, which is very similar to
that of \citet{chazottes} modulo the fact that the latter requires a
finite sample space.

These concentration inequalities are generalizations of the of
following inequality of \citet{mcdiarmid} commonly used in the i.i.d.\
setting.
\begin{theorem}[\citet{mcdiarmid}, 6.10]
\label{thm:mcdiarmid}
Let $S = (Z_1,\ldots,Z_m)$ be a sequence of random variables, each taking
values in the set $Z$, then for any measurable function $\Phi : Z^m \to
\Rset$ that satisfies the following, $\forall i \in {1,\ldots,m}, \forall
z_i,z_i' \in Z$,
\begin{equation*}
\left| \E_S \Big[ \Phi(S) \big| Z_1 = z_1,\ldots, Z_i = z_i \Big] - 
\E_S \Big[ \Phi(S) \big| Z_1 = z_1,\ldots, Z_i = z_i' \Big] \right| \leq
c_i,
\end{equation*}
for constants $c_i$.
Then, for all $\e > 0$,
\begin{equation*}
\Pr[ | \Phi - \E[\Phi] \geq \e ] 
\leq 2 \exp \left( \frac{-2 \e^2}{\sum_{i=1}^m c_i^2} \right).
\end{equation*}
\end{theorem}

In the i.i.d.\ scenario, the requirement to produce the constants $c_i$
simply translates into a Lipschitz condition on the function $\Phi$.
Theorem 5.1 of \citet{leo} bounds precisely this quantity as
follows,\footnote{
	We should note that original bound is expressed in terms of
	$\eta$-mixing coefficients. To simplify presentation, we are
	adapting it to the case of stationary $\varphi$-mixing sequences by
	using the following straightforward inequality for a stationary
	process: $2 \varphi(j - i) \geq \eta_{ij}$. Furthermore, the bound
	presented in \citet{leo} holds when the sample space is countable,
	it is extended to the continuous case in \citet{leo-thesis}.
}
\begin{equation}
\label{eq:martingale_diff}
 c_i \leq 1 + 2 \sum_{k=1}^{m - i} \varphi(k).
\end{equation}

Given the bound in Equation \ref{eq:martingale_diff}, the
concentration bound of McDiarmid can be restated as follows, making it
easily accessible to $\varphi$-mixing distributions.
\begin{theorem}[\citet*{leo}]
\label{th:leo}
Let $\Phi: Z^m \to \mathbb{R}$ be a measurable function. If $\Phi$ is
$l$-Lipschitz with respect to the Hamming metric for some $l > 0$, then the
following holds for all $\e > 0$:
\begin{equation}
\Pr_{Z}[ |\Phi(Z) - \E[\Phi(Z)]| > \e] \leq 2 \exp \left ( { -
2 \e^2 \over  m l^2 ||\Delta_m||_\infty^2 } \right ),
\end{equation}
where $||\Delta_m||_\infty \leq 1 + 2 \displaystyle \sum_{k = 1}^m
\varphi(k)$.
\end{theorem}
It should be pointed out that the statement of the theorem in this
paper is improved by a factor of $4$ in the exponent, from the one
stated in \citet{leo} Theorem 1.1.  This can be achieve
straightforwardly by following the same steps as in the proof by
\citet{leo} and making use of the general form of McDiarmid's
inequality (Theorem~\ref{thm:mcdiarmid}) as opposed to Azuma's
inequality.

This section presents several theorems that constitute the main
results of this paper. The following theorem is constructed form the
bounds shown in the previous three sections.

\begin{theorem}[General Non-i.i.d.\ Stability Bound]
\label{th:main1}
Let $h_S$ denote the hypothesis returned by a $\sbeta$-stable algorithm
trained on a sample $S$ drawn from a $\varphi$-mixing stationary
distribution and let $c$ be a measurable non-negative cost function upper
bounded by $M > 0$, then for any $b \in [0, m]$ and any $\e > 0$, the
following generalization bound holds
{\small
\begin{equation*}
\Pr_S \left[\left|R(h_S) - \h R(h_S)\right| >  \e + (6 b + 1) \sbeta + 6 M
\varphi(b) \right] \leq 2 \exp \left( \frac{ -2 \e^2 (1 + 2 \sum_{i=1}^m \varphi(i))^{-2}
} {m ((b + 1) 2 \sbeta + 2 M \varphi(b) + M/m)^2  } \right).
\end{equation*}}
\end{theorem}
\begin{proof}
The theorem follows directly the application of
Lemma~\ref{lemma:phi-bound} and Lemma~\ref{lemma:E-bound} to
Theorem~\ref{th:leo}.
\end{proof}
The theorem gives a general stability bound for $\varphi$-mixing
stationary sequences. If we further assume that the sequence is
algebraically $\varphi$-mixing, that is for all $k$, $\varphi(k) = \varphi_0
k^{-r}$ for some $r > 1$, then we can solve for the value of $b$ to
optimize the bound.\\

\begin{theorem}[Non-i.i.d.\ Stability Bound for Algebraically Mixing Sequences]
\label{th:main2}
Let $h_S$ denote the hypothesis returned by a $\sbeta$-stable
algorithm trained on a sample $S$ drawn from an algebraically
$\varphi$-mixing stationary distribution, $\varphi(k) = \varphi_0 k^{-r}$
with $r > 1$ and let $c$ be a measurable non-negative cost function
upper bounded by $M > 0$, then for any
$\e > 0$, the following generalization bound holds
{\small
\begin{equation*}
\Pr_S \left[\left|R(h_S) - \h R(h_S)\right| >  \e + \sbeta + (r+1) 6
M  \varphi(b) \right] \leq 2 \exp \left( \frac{ -2 \e^2 (1 + 2 \varphi_0 r/(r - 1))^{-2} 
}{m (2 \sbeta + (r+1) 2 M \varphi(b) + M/m)^2  } \right),
\end{equation*}}
where $\varphi(b) = \varphi_0 \left ( {\sbeta \over r \varphi_0 M}
\right)^{r / (r + 1)}$.
\end{theorem}
\begin{proof}
For an algebraically mixing sequence, the value of $b$ minimizing the
bound of Theorem~\ref{th:main1} satisfies $\sbeta b = r M \varphi(b)$,
which gives $b = \left ( {\sbeta \over r \varphi_0 M} \right)^{-1 / (r +
1)}$ and $\varphi(b) = \varphi_0 \left ( {\sbeta \over r \varphi_0 M}
\right)^{r / (r + 1)}$. The following term can be bounded as
\begin{eqnarray*}
1 + 2 \sum_{i=1}^m \varphi(i)  
= 1 + 2 \sum_{i=1}^m \varphi_0 i^{-r}
&\leq& 1 + 2 \varphi_0 \left(1 + \int_1^m i^{-r} di \right) 
= 1 + 2 \varphi_0 \left(1 + {m^{1-r} - 1\over 1 - r} \right ).
\end{eqnarray*}
Using the assumption $r > 1$, we upper bound $m^{1-r}$ with $1$ and
find that,
\begin{eqnarray*}
1 + 2 \varphi_0 \left(1 + {m^{1-r} - 1 \over 1 - r} \right )
&\leq& 1 + 2 \varphi_0 \left(1 + {1 \over r - 1} \right ) = 1 + {2 \varphi_0 r  \over r - 1}.
\end{eqnarray*}
Plugging in this value and the minimizing value of $b$ in the bound of
Theorem~\ref{th:main1} yields the statement of the theorem.
\end{proof}
In the case of a zero mixing coefficient ($\varphi = 0$ and $b = 0$),
the bounds of Theorem~\ref{th:main1} 
coincide with the i.i.d.\ stability bound of \citep{bousquet-jmlr}.  In
order for the right-hand side of these bounds to converge, we must
have $\sbeta = o(1/\sqrt{m})$ and $\varphi(b) = o(1 / \sqrt{m})$. For
several general classes of algorithms, $\sbeta \leq O(1/m)$
\citep{bousquet-jmlr}. In the case of algebraically mixing sequences
with $r > 1$, as assumed in Theorem~\ref{th:main2}, $\sbeta \leq O(1/m)$
implies $\varphi(b) = \varphi_0 (\sbeta / (r \varphi_0 M))^{(r / (r + 1))} <
O(1 / \sqrt{m})$. The next section illustrates the application of
Theorem~\ref{th:main2} to several general classes of algorithms.

We now present the application of our stability bounds to several
algorithms in the case of an algebraically mixing sequence. We make use of
the stability analysis found in \citet*{bousquet-jmlr}, which allows us to
apply our bounds in the case of kernel regularized algorithms, $k$-local
rules and relative entropy regularization.

\subsection{Applications}
\label{sec:apps}

\subsubsection{Kernel Regularized Algorithms}

Here we apply our bounds to a family of algorithms based on the
minimization of a regularized objective function based on the norm
$\|\cdot\|_K$ in a reproducing kernel Hilbert space, where $K$ is a
positive definite symmetric kernel:
\begin{equation}
\label{eq:kernel_reg}
\argmin_{h \in H} {1 \over m} \sum_{i=1}^m c(h,z_i) + \lambda \|h\|_K^2.
\end{equation}
The application of our bound is possible, under some general conditions,
since kernel regularized algorithms are stable with $\sbeta \leq O(1/m)$
\citep{bousquet-jmlr}.  Here we briefly reproduce the proof of this
$\sbeta$-stability for the sake of completeness; first we introduce
some needed terminology.

We will assume that the cost function $c$ is \emph{$\sigma$-admissible},
that is there exists $\sigma \in \Rset_+$ such that for any two hypotheses
$h, h' \in H$ and for all $z = (x, y) \in X \times Y$,
\begin{equation}
|c(h, z) - c(h', z)| \leq \sigma |h(x) - h'(x)|.
\end{equation}
This assumption holds for the quadratic cost and most other cost
functions when the hypothesis set and the set of output labels are
bounded by some $M \in \Rset_+$: $\forall h \in H, \forall x \in X,
|h(x)| \leq M$ and $\forall y \in Y, |y| \leq M$. We will also assume
that $c$ is differentiable. This assumption is in fact not necessary
and all of our results hold without it, but it makes the presentation
simpler.

We denote by $B_F$ the Bregman divergence associated to a convex
function $F$: $B_F(f \Arrowvert g) = F(f) - F(g) - \dotp{f - g}{\nabla
F(g)}$.  In what follows, it will be helpful to define $F$ as the
objective function of a general regularization based algorithm,
\begin{equation}
	F_S(h) = \h R_S(h) + \lambda N(h),
\end{equation}
where $\h R_S$ is the empirical error as measured on the sample $S$,
$N: H \to \Rset^+$ is a regularization function and $\lambda > 0$ is
the usual trade-off parameter.  Finally, we shall use the shorthand
$\Delta h = h' - h$.

\begin{lemma}[\citet{bousquet-jmlr}]
\label{lem:stability}
A kernel regularized learning algorithm, (\ref{eq:kernel_reg}), with
bounded kernel $K(x,x) \leq \kappa < \infty$ and $\sigma$-admissible cost
function, is $\sbeta$-stable with coefficient,
\begin{equation*}
	\sbeta \leq \frac{\sigma^2 \kappa^2}{m \lambda}
\end{equation*}
\end{lemma}
\begin{proof}
Let $h$ and $h'$ be the minimizers of $F_S$ and $F_S'$ respectively
where $S$ and $S'$ differ in the first coordinate (choice of
coordinate is without loss of generality), then,
\begin{equation}
\label{eq:claim}
B_N(h' \Arrowvert h) + B_N(h \Arrowvert h') \leq \frac{2 \sigma}{m \lambda}
\sup_{x \in S}|\Delta h(x)|.
\end{equation}
To see this, we notice that since $B_{F} = B_{\widehat R} + \lambda B_{N}$,
and since a Bregman divergence is non-negative,
\begin{equation*}
\lambda \bigl(B_N(h' \Arrowvert h) + B_N(h \Arrowvert h')\bigr) \leq
B_{F_S}(h' \Arrowvert h) + B_{F_{S'}}(h \Arrowvert h').
\end{equation*}
By the definition of $h$ and $h'$ as the minimizers of $F_S$ and
$F_{S'}$,
\begin{equation*}
  B_{F_S}(h' \Arrowvert h) + B_{F_{S'}}(h \Arrowvert h') =\\ \widehat
  R_{F_S}(h') - \widehat R_{F_S}(h) + \widehat R_{F_{S'}}(h) -
  \widehat R_{F_{S'}}(h').
\end{equation*}
Finally, by the $\sigma$-admissibility of the cost function $c$ and the
definition of $S$ and $S'$,
\begin{eqnarray*}
\lambda \bigl(B_N(h' \Arrowvert h) + B_N(h \Arrowvert h')\bigr)
& \leq & \widehat R_{F_S}(h') - \widehat R_{F_S}(h) + \widehat R_{F_{S'}}(h)
- \widehat R_{F_{S'}}(h') \\
& = & \frac{1}{m} \biggl[ c(h',z_1) - c(h,z_1) +
c(h,z_1') - c(h',z_1') \biggr] \\
& \leq  & \frac{1}{m} \biggl[ \sigma |\Delta h(x_1)| + \sigma |\Delta
h(x_1')| \biggr] \\
& \leq & \frac{2 \sigma}{m}   \sup_{x \in S} |\Delta h(x)|,
\end{eqnarray*}
which establishes (\ref{eq:claim}).

Now, if we consider $N(\cdot) = \norm{\cdot}_K^2$, we have $B_N(h'
\Arrowvert h) = \norm{h' - h}_K^2$, thus $B_N(h' \Arrowvert h) + B_N(h
\Arrowvert h') = 2 \norm{\Delta h}_K^2$ and by (\ref{eq:claim}) and the
reproducing kernel property,
\begin{equation*}
\begin{split}
			2 \norm{\Delta h}_K^2 & \leq {2 \sigma \over m \lambda} \sup_{x \in S}
			|\Delta h(x)|\\
      & \leq {2 \sigma \over m \lambda} \kappa ||\Delta h||_K.
\end{split}
\end{equation*}
Thus $\norm{\Delta h}_K \leq {\sigma \kappa \over m \lambda}$. And using
the $\sigma$-admissibility of $c$ and the kernel reproducing property we
get,
\begin{equation*}
\forall z \in X \times Y, \abs{c(h', z) - c(h, z)} \leq \sigma \abs{\Delta h (x)} \leq \kappa \sigma \norm{\Delta h}_K.
\end{equation*}
Therefore, 
\begin{equation*}
\small  \forall z \in X \times Y, \abs{c(h', z) - c(h, z)} \leq {\sigma^2
\kappa^2 \over m \lambda},
\end{equation*}
which completes the proof.
\end{proof}

Three specific instances of kernel regularization algorithms are SVR,
for which the cost function is based on the $\e$-insensitive cost:
\begin{equation}
c(h,z) = |h(x) - y|_{\e} = 
\begin{cases}
0 & \text{if $|h(x) - y| \leq \e$},\\
|h(x) - y| - \e & \text{otherwise.}
\end{cases}
\end{equation}
Kernel Ridge Regression \citep{krr}, for which 
\begin{equation}
c(h,z) = (h(x) - y)^2 \;,
\end{equation}
and finally Support Vector Machines with the hinge-loss,
\begin{equation}
c(h,z) = 
\begin{cases}
0 & \text{if  } 1 - y h(x) \leq 0, \\
1 - y h(x) & \text{if  } 0 \leq y h(x) < 1, \\
1 & \text{if  } y h(x) < 0.
\end{cases}
\end{equation}

We note that for kernel regularization algorithms, as pointed out in
\citet*[Lemma 23]{bousquet-jmlr}, a bound on the labels immediately
implies a bound on the output of the hypothesis produced by equation
(\ref{eq:kernel_reg}). We formally state this lemma below.

\begin{lemma}
\label{lemma:bounded_labels}
Let $h^*$ be the solution to equation (\ref{eq:kernel_reg}), let $c$ be a
cost function and let $B(\cdot)$ be a real-valued function such that
$\forall y \in \{ y \mid \exists x \in X, \exists h \in H, y = h(x) \},
\forall y' \in Y$,
\begin{equation*}
c(y,y') \leq B(y).
\end{equation*}
Then, the output of $h^*$ is bounded as follows,
\begin{equation*}
	\forall x \in X, |h^*(x)| \leq \kappa \sqrt{\frac{B(0)}{\lambda}},
\end{equation*}
where $\lambda$ is the regularization parameter, and $\kappa^2 \geq K(x,x)$
for all $x \in X$.
\end{lemma}
\begin{proof}
Let $F(h) = {1 \over m} \sum_{i=1}^m c(h,z_i) + \lambda \|h\|_K^2$ and let
{\bf 0} be the zero hypothesis, then by definition of $F$ and $h^*$,
$$ \lambda \|h^*\|_K^2 \leq F(h^*) \leq F({\bf 0}) \leq B(0).$$
Then, using the reproducing kernel property and the Cauchy-Schwartz
inequality we note,
$$\forall x \in X,  
|h^*(x)| = \langle h^*,K(x,\cdot) \rangle \leq \|h^*\|_K \sqrt{K(x,x)}
\leq \kappa \|h^*\|_K.$$
Combining the two inequalities produces the result.
\end{proof}
We note that in \citet{bousquet-jmlr}, the following the bound is also
stated: $c(h^*(x),y') \leq B(\kappa \sqrt{B(0)/\lambda})$.
However, when later applied it seems the authors use an incorrect
upper bound function $B(\cdot)$, which we remedy in the following.

\begin{corollary}
\label{cor:kernel_reg}
Assume a bounded output $Y = [0,B]$, for some $B > 0$, and assume that
$K(x,x) \leq \kappa^2$ for all $x$ for some $\kappa > 0$. Let $h_S$ denote
the hypothesis returned by the algorithm when trained on a sample $S$ drawn
from an algebraically $\varphi$-mixing stationary distribution.
Let $u = r/(r + 1) \in [\frac{1}{2}, 1]$, $M' = 2 (r + 1) \varphi_0 M
/ (r \varphi_0 M)^u$, and $\varphi'_0 = (1 + 2\varphi_0 r / (r-1))$.
Then, with probability at least $1 - \delta$, the following
generalization bounds hold for\\

\begin{enumerate}
  \renewcommand{\labelenumi}{\alph{enumi}.}
\item Support Vector Machines (SVM, with hinge-loss)
{\small
\begin{equation*}
R(h_S)
\leq \h R(h_S)
+ \frac{\kappa^2}{\lambda m}
+ \left( \frac{\kappa^2}{\lambda} \right)^{u}\frac{3 M'}{m^u}
+ \varphi'_0 
\biggl(1  + 
\frac{\kappa^2}{\lambda} + 
\biggl(\frac{\kappa^2}{\lambda}\biggr)^{u}  
\frac{M'}{m^{u - 1}} \biggr)
\sqrt{\frac{2 \log(2/\delta)}{m}},
\end{equation*}
}
where $M = 1$.
\item Support Vector Regression (SVR):
{\small
\begin{equation*}
R(h_S)
\leq \h R(h_S)
+ \frac{\kappa^2}{\lambda m}
+ \left( \frac{\kappa^2}{\lambda} \right)^{u}\frac{3 M'}{m^u}
+ \varphi'_0 
\biggl(M  + 
\frac{\kappa^2}{\lambda} + 
\biggl(\frac{\kappa^2}{\lambda}\biggr)^{u}  
\frac{M'}{m^{u - 1}} \biggr)
\sqrt{\frac{2 \log(2/\delta)}{m}},
\end{equation*}
}
where $M = \kappa \sqrt{\frac{B}{\lambda}} + B$. \\
\item Kernel Ridge Regression (KRR):
{\small
\begin{equation*}
R(h_S)
\leq \h R(h_S)
+ \frac{4 \kappa^2 B^2}{\lambda m}
+ \left( \frac{4 \kappa^2 B^2}{\lambda}\right)^{u}\frac{3 M'}{m^u}
+ \varphi'_0 \biggl(M 
+ \frac{4 \kappa^2 B^2}{\lambda}
+ \biggl(\frac{4 \kappa^2 B^2}{ \lambda}\biggr)^{u}\frac{M'}{m^{u - 1}} 
 \biggr) \sqrt{\frac{2 \log(2/\delta)}{m}},
\end{equation*}
}
where $M = \kappa^2 B^2/\lambda + B^2$.  \\
\end{enumerate}
\end{corollary}
\begin{proof}
For SVM, the hinge-loss is $1$-admissible giving $\sbeta \leq
\kappa^2 / (\lambda m)$, and the cost function is clearly bounded by $M=1$.

Similarly, SVR has a loss function that is $1$-admissible, thus,
applying Lemma \ref{lem:stability} gives us $\sbeta \leq \kappa^2 /
(\lambda m)$. Using Lemma \ref{lemma:bounded_labels}, with $B(0) = B$,
we can bound the loss as follows: $\forall x \in X, y \in Y, |h^*(x) -
y| \leq \kappa \sqrt{\frac{B}{\lambda}} + B$.  

Finally for KRR, we have a loss function that is $2B$-admissible and
again using Lemma \ref{lem:stability} $\sbeta \leq 4 \kappa^2 B^2 /
(\lambda m)$. Again, applying Lemma \ref{lemma:bounded_labels} with $B(0) =
B^2$ and $\forall x \in X, y\in Y, (h^*(x) - y)^2 \leq \kappa^2 B^2 /
\lambda + B^2$ \ignore{[[these are harsh bounds, a more careful one for
example is $(a-b)^p \leq \max(a,b)^p$ since $a,b\geq0$]]}. 

Plugging these values into the bound of Theorem~\ref{th:main2} and
setting the right-hand side to $\delta$ yields the statement of the
corollary.
\end{proof}

\ignore{  
\subsubsection{$k$-Local Rule Algorithms}

The next class of algorithms we will apply our bounds to are those which
determine the label of a point $x$ based on the labels of the $k$ nearby 
points from the training set. In particular, the $k$-Nearest Neighbor
algorithm produces a label for $x$ based on the label of the majority of
the $k$ closest neighbors. \citet*{devroye-knn} have shown the $k$-Nearest
Neighbors algorithm, with $\{0,1\}$-loss, has stability coefficient, 
$$\sbeta \leq \frac4m \sqrt{\frac{k}{2\pi}}.$$
We are now able to apply Theorem \ref{th:main2}.
\begin{corollary}
Let $h_S$ denote the hypothesis defined by the majority $k$-Nearest
Neighbor rule over a sample $S$ drawn from an algebraically
$\varphi$-mixing stationary distribution. If the loss is defined in terms
of the $\{0,1\}$-loss, then the following bound holds with probability $1 -
\delta$,
\begin{equation*}
R(h_S)
\leq \h R(h_S)
+ \frac{4}{m} \sqrt{\frac{k}{2 \pi}}
+ \frac{3 M' k^{u/2}}{m^u}
+ \varphi'_0 
\biggl(M  + 
4 \sqrt{\frac{k}{2 \pi}} + 
\frac{M' k^{u/2}}{m^{u - 1}} \biggr)
\sqrt{\frac{2 \log(2/\delta)}{m}},
\end{equation*}
where $u = r/(r + 1) \in [\frac{1}{2}, 1]$, $M' = 2 (r + 1) \varphi_0 M 4^u /
(r \varphi_0 M \sqrt{2 \pi})^u$, and $\varphi'_0 = (1 + 2\varphi_0 r / (r-1))$.
\end{corollary}
}

\subsubsection{Relative Entropy Regularized Algorithms}

In this section we apply Theorem \ref{th:main2} to algorithms that
produce a hypothesis $h$ that is a convex combination of base hypotheses
$h_\theta \in H$ which are parameterized by $\theta \in \Theta$. Thus, we
wish to learn a weighting function $g \in G : \Theta \to \Rset$ that is a
solution to the following optimization,
\begin{equation}
\label{eq:entropy_reg}
 \argmin_{g \in G} \frac1m \sum_{i=1}^m c(g,z_i) + \lambda D(g \| g_0),
\end{equation}
where the cost function $c : G \times Z \to \Rset$ is defined in term of a
second internal cost function $c' : H \times Z \to \Rset$:
\begin{equation*}
 c(g, z) = \int_{\Theta} c'(h_\theta, z) g(\theta) d\theta, 
\end{equation*}
and where $D$ is the Kullback-Leibler divergence or relative entropy
regularizer (with respect to some fixed distribution $g_0$):
\begin{equation*}
 D(g \| g_0) = \int_\Theta g(\theta) \ln{\frac{g(\theta)}{g_0(\theta)}}
d\theta.
\end{equation*}

It has been shown, \citep[Theorem 24]{bousquet-jmlr}, that an algorithm
satisfying equation \ref{eq:entropy_reg} and with bounded loss $c'(\cdot)
\leq M$, is $\sbeta$-stable with coefficient
\begin{equation*}
 \sbeta \leq \frac{M^2}{\lambda m}.
\end{equation*}

The application of our bounds, results in the following corollary.
\begin{corollary}
Let $h_S$ be the hypothesis produced by the optimization in
(\ref{eq:entropy_reg}), with internal cost function $c'$ bounded by $M$.
Then with probability at least $1 - \delta$, 
\begin{equation*}
R(h_S)
\leq \h R(h_S)
+ \frac{M^2}{\lambda m}
+ \frac{3 M'}{\lambda^u m^u}
+ \varphi'_0 
\biggl(M  + 
\frac{M^2}{\lambda} + 
\frac{M'}{\lambda^u m^{u - 1}} \biggr)
\sqrt{\frac{2 \log(2/\delta)}{m}},
\end{equation*}
where $u = r/(r + 1) \in [\frac{1}{2}, 1]$, $M' = 2 (r + 1) \varphi_0
M^{u+1} /
(r \varphi_0)^u$, and $\varphi'_0 = (1 + 2\varphi_0 r / (r-1))$.
\end{corollary}

\subsection{Discussion}

The results presented here are, to the best of our knowledge, the
first stability-based generalization bounds for the class of
algorithms just studied in a non-i.i.d.\ scenario.  These bounds are
non-trivial when the condition on the regularization $\lambda \gg
1/m^{1/2 - 1/r}$ parameter holds for all large values of $m$. This
condition coincides with the i.i.d.\ condition, in the limit, as $r$
tends to infinity.  The next section gives stability-based
generalization bounds that hold even in the scenario of $\beta$-mixing
sequences.

\section{$\beta$-Mixing Generalization Bounds}
\label{sec:beta-mixing}

In this section, we prove a stability-based generalization bound that
only requires the training sequence to be drawn from a stationary
$\beta$-mixing distribution. The bound is thus more general and covers
the $\varphi$-mixing case analyzed in the previous section. However,
unlike the $\varphi$-mixing case, the $\beta$-mixing bound presented
here is not a purely exponential bound. It contains an additive term,
which depends on the mixing coefficient.

As in the previous section, $\Phi(S)$ is defined by $\Phi(S) = R(h_S)
- \h R(h_S)$. To simplify the presentation, here, we will define the
generalization error of $h_S$ by $R(h_S) = \E_z[c(h_S, z)]$. Thus,
test samples are assumed independent of $S$. By
Lemma~\ref{lemma:risk}, this can be assumed modulo the additional term
$b \sbeta + M \beta(b)$, for a cost function bounded by $M$. Note that
for any block of points $Z = z_1 \ldots z_k$ drawn independently of
$S$, the following equality 
\begin{equation}
  \E_Z\biggl[\frac{1}{|Z|}\sum_{z \in Z}c(h_S, z)\biggr] = \frac{1}{k}\sum_{i = 1}^k
\E_Z[c(h_S, z_i)] = \frac{1}{k}\sum_{i = 1}^k
\E_{z_i}[c(h_S, z_i)] = \E_z[c(h_S, z)]
\end{equation}
holds since, by stationarity, $\E_{z_i}[c(h_S, z_i)] = \E_{z_j}[c(h_S,
z_j)]$ for all $1 \leq i, j \leq k$. Thus, $R(h_S) =
\E_Z\bigl[\frac{1}{|Z|}\sum_{z \in Z}c(h_S, z)\bigr]$ for any such
block $Z$.  For convenience, we will extend the cost function $c$ to
blocks as follows:
\begin{equation}
c(h, Z) = \frac{1}{|Z|} \sum_{z \in Z} c(h, z).
\end{equation}
With this notation, $R(h_S) = \E_Z[c(h_S, Z)]$ for any block drawn
independently of $S$, regardless of the size of $Z$.

To derive a generalization bound for the $\beta$-mixing scenario, we
will apply McDiarmid's inequality to $\Phi$ defined over a sequence of
independent blocks.  The independent blocks we will be considering are
non-symmetric and thus more general than those considered by previous
authors \citep{yu, meir, lozano}.

From a sample $S$ made of a sequence of $m$ points, we construct two
sequences of blocks $S_a$ and $S_b$, each containing $\mu$ blocks.
Each block in $S_a$ contains $a$ points and each block $S_b$ in
contains $b$ points. $S_a$ and $S_b$ form a partitioning of $S$; for
any $a, b \in [0,m]$ such that $(a + b) \mu = m$, they are defined
precisely as follows:
\begin{equation}
\label{eq:block_def}
\begin{split}
&	S_a = (Z^{(a)}_1, \ldots,Z^{(a)}_\mu), \text{ with } Z^{(a)}_i =
	z_{(i-1)(a + b) + 1}, \ldots, z_{(i-1)(a + b) + a}\\
&	S_b = (Z^{(b)}_1, \ldots, Z^{(b)}_\mu), \text{ with } Z^{(b)}_i =
	z_{(i-1)(a + b) + a + 1}, \ldots, z_{(i-1)(a + b) + a + b},
\end{split}
\end{equation}
for all $i \in [1, \mu]$. We shall consider similarly sequences of
i.i.d.\ blocks $\tl Z_i^{a}$ and $\tl Z_i^{b}$, $i \in [1, \mu]$, such
that the points within each block are drawn according to the same
original $\beta$-mixing distribution and shall denote by $\tl S_a$ the
block sequence $(\tl Z^{(a)}_1, \ldots, \tl Z^{(a)}_\mu)$. 
In preparation for the application of McDiarmid's inequality, we give
a bound on the expectation of $\Phi(\tl S_a)$.  Since the expectation
is taken over a sequence of i.i.d.\ blocks, this brings us to a
situation similar to the i.i.d.\ scenario analyzed by
\citet{bousquet-jmlr}, with the exception that we are dealing with
i.i.d.\ blocks instead of i.i.d.\ points.
\begin{lemma}
\label{lem:exp_bound}
Let $\tl S_a$ be an independent block sequence as defined above, then the
following bound holds for the expectation of $|\Phi (\tl S_a)|$:
\begin{equation*}
\E_{\tl S_a} [|\Phi (\tl S_a)|] \leq a \sbeta.
\end{equation*}
\end{lemma}
\begin{proof}
  Since the blocks $\tl Z^{(a)}$ are independent, we can replace any
  one of them with any other block $Z$ drawn from the same
  distribution. However, changing the training set also changes the
  hypothesis, in a limited way. This is shown precisely below,
\begin{eqnarray*}
	\E_{\tl S_a} [|\Phi (\tl S_a)|] 
	& = & \E_{\tl S_a} \left[ \Big| \frac{1}{\mu} \sum_{i=1}^\mu c(h_{\tl S_a},
	\tl Z^{(a)}_i) - \E_{Z} [ c(h_{\tl S_a}, Z) ] \Big| \right] \\
	& \leq & \E_{\tl S_a, Z} \left[ \Big| \frac{1}{\mu} \sum_{i = 1}^\mu c(h_{\tl S_a},
	\tl Z^{(a)}_i) - c(h_{\tl S_a}, Z)  \Big| \right] \\
	& = & \E_{\tl S_a, Z} \left[ \Big| \frac{1}{\mu} \sum_{i = 1}^\mu c(h_{\tl S_a^i}, Z) -
	c(h_{\tl S_a}, Z)  \Big| \right],
\end{eqnarray*}
where $\tl S_a^i$ corresponds to the block sequence $\tl S_a$ obtained
by replacing the $i$th block with $Z$. The inequality holds through
the use of Jensen's inequality. The $\sbeta$-stability of the learning
algorithm gives
\begin{equation*}
 \E_{\tl S_a, Z} \left[ \frac{1}{\mu} \Big| \sum_{i = 1}^\mu c(h_{\tl S_a^i}, Z) -
 c(h_{\tl S_a}, Z)  \Big| \right] \leq \E_{\tl S_a, Z} \left[ \frac{1}{\mu} \sum_{i =
 1}^\mu a \sbeta  \right] \leq a \sbeta.
\end{equation*}
\end{proof}
We now relate the non-i.i.d.\ event $\Pr[\Phi(S) \geq \e]$ to an
independent block sequence event to which we can apply McDiarmid's
inequality.
\begin{lemma}
\label{lem:prep}
Assume a $\sbeta$-algorithm. Then, for a sample $S$ drawn from a
stationary $\beta$-mixing distribution, the following bound holds,
\begin{equation}
\label{eq:lemma_prep}
	\Pr_S[|\Phi(S)| \geq \e] \leq \Pr_{\tl S_a} \big[|\Phi(\tl S_a)| -
\E[|\Phi(\tl S_a)|] \geq \e_0'\big] + (\mu - 1)\beta(b), 
\end{equation}
where $\e_0' = \e - \frac{\mu b M}{m} - 2 \mu b \sbeta -
\E_{\tl S_a'}[|\Phi(\tl S_a')|]$.
\end{lemma}
\begin{proof}
  The proof consists of first rewriting the event in terms of $S_a$
  and $S_b$ and bounding the error on the points in $S_b$ in a trivial
  manner. This can be afforded since $b$ will be eventually chosen to
  be small. Since $|\E_{Z'} [c(h_S,Z')] - c(h_S,z')| \leq M$ for
  any $z' \in S_b$, we can write
\begin{align*}
\Pr_S[|\Phi(S)| \geq \e] & = \Pr_S[ |R(h_S) - \h R(h_S)| \geq \e ] \\
& = \Pr_S \bigg[ \frac{1}{m} \Big| \sum_{z \in S} \E_Z [c(h_S,Z)] - c(h_S,z)
	 \Big| \geq \e \bigg] \\
& \leq \Pr_S \bigg[ \frac{1}{m} \Big| \sum_{z \in S_a} \E_Z [c(h_S,Z)] - c(h_S,z)
	 \Big|  + \frac{1}{m} \Big| \sum_{z' \in S_b} \E_{Z'} [c(h_S,Z')] -
			c(h_S,z') \Big| \geq \e \bigg] \\
& \leq \Pr_S \bigg[ \frac{1}{m} \Big| \sum_{z \in S_a} \E_Z [c(h_S,Z)] -
	c(h_S,z) \Big| + \frac{\mu b M}{m} \geq \e \bigg].
\end{align*}
By $\sbeta$-stability and $\mu a / m \leq 1$, this last term can be bounded as follows
\begin{multline*}
\Pr_S \bigg[ \frac{1}{m} \Big| \sum_{z \in S_a} \E_Z [c(h_S,Z)] -
	c(h_S,z) \Big| + \frac{\mu b M}{m} \geq \e \bigg] \leq \\
\Pr_{S_a} \bigg[ \frac{1}{\mu a} \Big| \sum_{z \in S_a}  \E_Z [c(h_{S_a},Z)] -
	c(h_{S_a},z) \Big| + \frac{\mu b M}{m} + 2 \mu b \sbeta \geq \e \bigg].
\end{multline*}
The right-hand side can be rewritten in terms of $\Phi$ and bounded in
terms of a $\beta$-mixing coefficient:
\begin{multline*}
\Pr_{S_a} \bigg[ \frac{1}{\mu a} \Big| \sum_{z \in S_a} \E_Z [c(h_{S_a},Z)] -
	c(h_{S_a},z) \Big| + \frac{\mu b M}{m} + 2 \mu b \sbeta \geq \e \bigg]\\
\begin{split}
= & \Pr_{S_a} \bigg[ |\Phi(S_a)| + \frac{\mu b M}{m} + 2 \mu b \sbeta \geq
\e \bigg] \\
\leq & \Pr_{\tl S_a} \bigg[ |\Phi(\tl S_a)| + \frac{\mu b M}{m} + 2 \mu b \sbeta \geq \e \bigg]  +
(\mu-1) \beta(b),
\end{split}
\end{multline*}
by applying Lemma~\ref{lemma:blocks} to the indicator function of the
event $\set{|\Phi(S_a)| + \frac{\mu b M}{m} + 2 \mu b \sbeta \geq
  \e}$. Since $\E_{\tl S_a'}[|\Phi(\tl S_a')|]$ is a constant, the
probability in this last term can be rewritten as
\begin{equation*}
\begin{split}
\Pr_{\tl S_a} \bigg[ |\Phi(\tl S_a)|  + \frac{\mu b M}{m} + 2 \mu b \sbeta] \geq \e \bigg]
= & \Pr_{\tl S_a} \bigg[ |\Phi(\tl S_a)| - \E_{\tl S_a'}[|\Phi(\tl S_a')|]
+ \frac{\mu b M}{m} + 2 \mu b \sbeta] \geq \e - \E_{\tl S_a'}[|\Phi(\tl S_a')|] \bigg]\\
= & \Pr_{\tl S_a} \bigg[|\Phi(\tl S_a)| - \E_{\tl S_a'}[|\Phi(\tl S_a')|] \geq \e_0' \bigg],
\end{split}
\end{equation*}
which ends the proof of the lemma.
\end{proof}
The last two lemmas will help us prove the main result of this section formulated
in the following theorem.
\begin{theorem}
  Assume a $\sbeta$-stable algorithm and let $\e'$ denote $\e -
  \frac{\mu b M}{m} - 2 \mu b \sbeta - a \sbeta$ as in
  Lemma~\ref{lem:prep}. Then, for any sample $S$ of size $m$ drawn
  according to a stationary $\beta$-mixing distribution, any choice of
  the parameters $a, b, \mu > 0$ such that $(a + b)
  \mu = m$, and $\e \geq 0$ such that $\e' \geq 0$, the following
  generalization bound holds:
\begin{equation*}
\Pr_S \Big[ |R(h_S) - \h R(h_S)| \geq \e \Big]
\leq \exp \left( \frac{-2 \e'^2 m}{\big(2 a\sbeta m + (a + b) M
\big)^2}\right) + (\mu - 1) \beta(b).
\end{equation*}
\end{theorem}
\begin{proof}
To prove the statement of theorem, it suffices to bound the probability
term appearing in the right-hand side of Equation~\ref{eq:lemma_prep},
$\Pr_{\tl S_a} \big[|\Phi(\tl S_a)| - \E[|\Phi(\tl S_a)]| \geq \e_0'\big]$,
which is expressed only in terms of independent blocks. We can therefore
apply McDiarmid's inequality by viewing the blocks as i.i.d.\ ``points''.

To do so, we must bound the quantity $\big||\Phi(\tl S_a)| - |\Phi(\tl
S_a^i) |\big|$ where the sequence $S_a$ and $S_a^{i}$ differ in the $i$th
block.  We will bound separately the difference between the generalization
errors and empirical errors.\footnote{We drop the superscripts on $Z^{(a)}$
since we will not be considering the sequence $S_b$ in what follows.} The
difference in empirical errors can be bounded as follows using the bound on
the cost function $c$:
\begin{eqnarray*}
	|\h R(h_{S_a}) - \h R(h_{S_a^{i}})| 
	&=& \bigg| \frac{1}{\mu} \bigg[ \sum_{j \neq i} c(h_{S_a}, Z_j) -
	c(h_{S_a^{i}}, Z_j) \bigg] + \frac{1}{\mu} \big[ c(h_{S_a}, Z_i) -
	c(h_{S_a^{i}}, Z_i') \big] \bigg| \\
	& \leq & a \sbeta + \frac{M}{\mu}  = a \sbeta + \frac{(a + b) M}{m}.
\end{eqnarray*}
The difference in generalization error can be straightforwardly
bounded using $\sbeta$-stability:
\begin{eqnarray*}
	|R(h_{S_a}) - R(h_{S_a^{i}})| 
	= | \E_Z [ c(h_{S_a},Z) ] - \E_Z [ c(h_{S_a^{i}},Z) ] | 
	= | \E_Z [ c(h_{S_a},Z) -  c(h_{S_a^{i}},Z) ] | \leq a \sbeta.
\end{eqnarray*}
Using these bounds in conjunction with McDiarmid's inequality yields
\begin{align*}
\Pr_{\tl S_a}[ |\Phi(\tl S_a)| - \E_{\tl S_a'} [|\Phi(\tl S_a')|] \geq \e_0']
& \leq \exp \left( \frac{-2 \e_0'^2 m}{\big(2 a\sbeta m + (a + b) M
\big)^2} \right) \\
& \leq \exp \left( \frac{-2 \e'^2 m}{\big(2 a\sbeta m + (a + b) M \big)^2} \right).
\end{align*}
Note that to show the second inequality we make use of
Lemma~\ref{lem:exp_bound} to estabilish the fact that
\begin{equation*}
	\e_0' 
	=  \e - \frac{\mu b M}{m} - 2 \mu b \sbeta - \E_{\tl S_a'}[|\Phi(\tl S_a')|]
	\geq \e - \frac{\mu b M}{m} - 2 \mu b \sbeta - \alpha \sbeta
	= \e'.
\end{equation*}
Finally, we make use of Lemma~\ref{lem:prep} to establish the proof,
\begin{align*}
\Pr_S[|\Phi(S)| \geq \e] 
&\leq \Pr_{\tl S_a} \big[|\Phi(\tl S_a)| -
\E[|\Phi(\tl S_a)|] \geq \e_0'\big] + (\mu - 1)\beta(b) \\
& \leq  \exp \left( \frac{-2 \e'^2 m}{\big(2 a\sbeta m + (a + b) M \big)^2}
\right) + (\mu - 1) \beta(b).
\end{align*}
\end{proof}

In order to make use of the bounds, we must select the values of
parameters $b$ and $\mu$ ($a$ is then equal to $\mu / m - u$). There
is a trade-off between choosing large value for $b$, to ensure the
mixing term decreases, while choosing a large value of $\mu$, to
minimize the remaining terms of the bound. The exact choice of
parameters will depend on the type of mixing that is assumed (e.g.
algebraic or exponential).  In order to choose optimal parameters, it
will be useful to view the bound as it holds with high probability, in
the following corollary.
\begin{corollary}
\label{cor:whp}
Assume a $\sbeta$-stable algorithm and let $\delta'$ denote $\delta - (\mu
- 1) \beta(b)$. Then, for any sample $S$ of size $m$ drawn according to a
stationary $\beta$-mixing distribution, any choice of the parameters $a, b,
\mu > 0$ such that $(a + b) \mu = m$, and $\delta \geq 0$ such that
$\delta' \geq 0$, the following generalization bound holds with probability
at least $(1 - \delta)$:
\begin{equation*}
	|R(h_S) - \h R(h_S)|  < \sqrt{\frac{\log \left(1 / \delta') \right)}{2m}}
	\left( 2 a \sbeta m + M \frac{m}{\mu} \right)
	+ \mu b \left( \frac{M}{m} + 2 \sbeta \right) + a \sbeta
\end{equation*}
\end{corollary}

In the case of a fast mixing distribution, it is possible to select
the values of the parameters to retrieve a bound as in the i.i.d.\
case, i.e.  $|R(h_S) - \h R(h_S)| \in O \Big( m^{-\frac{1}{2}}
\sqrt{\log 1/\delta} \Big)$. In particular, for $\beta(b) \equiv 0$,
we can choose $a=0$, $b=1$ and $\mu=m$ to retrieve the i.i.d.\ bound
of \citet{bousquet}.

In the following, we will examine slower mixing algebraic
$\beta$-mixing distributions, which are thus not close to the i.i.d.\
scenario.  For algebraic mixing the mixing parameter is defined as
$\beta(b) = b^{-r}$. In that case, we wish to minimize the following
function in terms of $\mu$ and $b$.
\begin{equation}
\label{eq:shape}
s(\mu, b) = \frac{\mu}{b^r} 
		+ \frac{m^{3/2} \sbeta}{\mu}
		+ \frac{m^{1/2}}{\mu} 
		+ \mu b \left( \frac{1}{m} + \sbeta \right) .
\end{equation}
The first term of the function captures the condition on $\delta > (\mu +
1) \beta(b) \approx \mu / b^r$ and the remaining terms capture the shape of
the bound in Corollary~\ref{cor:whp}.

Setting the derivative with respect to each variable $\mu$ and $b$ to zero
and solving for each parameter results in the following expressions:
\begin{equation}
	b = C_r \sbetam^{-\frac{1}{r + 1}} ,
	\qquad
	\mu =	\frac{m^{3/4} \sbetam^\frac{1}{2(r+1)}} {\sqrt{C_r (1 + 1/r)}},
\end{equation}
where $\sbetam = (m^{-1} + \sbeta)$ and $C_r =
r^\frac{1}{r+1}$ is a constant defined by the parameter $r$.  

\ignore{
If we try to retrieve the i.i.d.\ case by letting $r \to \infty$ we see
that $b=1$ and $\mu=m^{3/4}$ (it can be verified that $C_r \to 1$), which
does not quite match our expectation (it should be that $\mu=m$).  This is
an artifact of having to choose $\mu$ and $b$ such that both $(\mu b) / m$
and $\beta(b)$ tend to $0$ as $m \to \infty$.  However, in the i.i.d.\
case, where $\beta(b) \equiv 0$, we are free to choose $b=0, a=1, \mu=m$,
where again we are able to retrieve the i.i.d.\ bound of
\citet{bousquet-jmlr}.
}

Now, assuming $\sbeta \in O(m^{-\alpha})$ for some $0 < \alpha \leq 1$, we
analyze the convergence behavior of Corollary~\ref{cor:whp}.  First, we
notice that the terms $b$ and $\mu$ have the following asymptotic behavior,
\begin{equation}
	b	\in O \left( m^\frac{\alpha}{r+1} \right) ,
	\qquad
	\mu \in O \left( m^{\frac{3}{4} - \frac{\alpha}{2(r + 1)}} \right) .
\end{equation}
Next, we consider the condition $\delta' > 0$ which is equivalent to,
\begin{equation}
\label{eq:delta_constraint}
	\delta > (\mu - 1) \beta(b) 
	\in O \bigg( m^{\frac{3}{4} - \alpha \big( 1 - \frac{1}{2(r+1)}
	\big)}\bigg).
\end{equation}
In order for the right-hand side of the inequality to converge, it must be
the case that $\alpha > \frac{3r + 3}{4r + 2}$. In particular, if
$\alpha=1$, as we have shown is the case for several algorithms in
Section~\ref{sec:apps}, then it suffices that $r > 1$.

Finally, in order to see how the bound itself converges, we study the
asymptotic behavior of the terms of Equation~\ref{eq:shape} (without the
first term, which corresponds to the quantity already analyzed in
Equation~\ref{eq:delta_constraint}):
\begin{equation}
		\underbrace{
		\frac{m^{3/2} \sbeta}{\mu}
		+ \mu b \sbeta
		}_{(a)}
		\underbrace{
		+ \frac{m^{1/2}}{\mu} 
		+ \frac{\mu b}{m} 
		}_{(b)}
		\in O \underbrace{ \bigg(
		m^{\frac{3}{4} - \alpha \big( 1 -  \frac{1}{2(r+1)} \big)}
		}_{(a)}
		\underbrace{
		+ ~m^{\frac{\alpha}{2(r + 1)} - \frac{1}{4}}
		\bigg)}_{(b)}.
\end{equation}
This expression can be further simplified by noticing that $(b) \leq (a)$
for all $0 < \alpha \leq 1$ (with equality at $\alpha=1$).  Thus, both the
bound and the condition on $\delta$ decrease asymptotically as the term in 
$(a)$, resulting in the following corollary.

\begin{corollary}
Assume a $\sbeta$-stable algorithm with $\sbeta \in O(m^{-1})$ and let $\delta' = \delta - m^{\frac{1}{2(r + 1)} - \frac{1}{4}}$. Then, for any sample $S$ of
size $m$ drawn according to a stationary algebraic $\beta$-mixing
distribution, and $\delta \geq 0$ such that $\delta' \geq 0$, the following
generalization bound holds with probability at least $(1 - \delta)$:
\begin{equation}
	|R(h_S) - \h R(h_S)|  < O \bigg(
			m^{\frac{1}{2(r + 1)} - \frac{1}{4}}
			\sqrt{\log(1/\delta')}
		\bigg).
\end{equation}
\end{corollary}
As in previous bounds $r>1$ is required for convergence. Furthermore,
as expected, a larger mixing parameter $r$ leads to a more favorable
bound.

\section{Conclusion}
\label{sec:conclusion}

We presented stability bounds for both $\varphi$-mixing and
$\beta$-mixing stationary sequences. Our bounds apply to large classes
of algorithms, including common algorithms such as SVR, KRR, and SVMs,
and extend to non-i.i.d.\ scenarios existing i.i.d.\ stability
bounds. Since they are algorithm-specific, these bounds can often be
tighter than other generalization bounds based on general complexity
measures for families of hypotheses. As in the i.i.d.\ case, weaker
notions of stability might help further improve and refine these
bounds.

Our bounds can be used to analyze the properties of stable
algorithms when used in the non-i.i.d\ settings studied. But, more
importantly, they can serve as a tool for the design of novel and
accurate learning algorithms. Of course, some mixing properties of the
distributions need to be known to take advantage of the information
supplied by our generalization bounds.  In some problems, it is
possible to estimate the shape of the mixing coefficients. This should
help devising such algorithms.

\acks{This work was partially funded by the New York State Office of Science
Technology and Academic Research (NYSTAR) and a Google Research Award.}

\vskip 0.2in
\bibliography{niidj}
\end{document}